\documentclass[runningheads]{llncs}
\usepackage{makeidx}
\usepackage{graphicx}
\usepackage{amsmath,amssymb}
\usepackage{color}

\usepackage{microtype}
\usepackage{subfigure}
\usepackage{booktabs}

\usepackage{amsmath,amsfonts,bm}

\def\eqref#1{equation~\ref{#1}}

\def\1{\bm{1}}

\DeclareMathAlphabet{\mathsfit}{\encodingdefault}{\sfdefault}{m}{sl}
\SetMathAlphabet{\mathsfit}{bold}{\encodingdefault}{\sfdefault}{bx}{n}

\usepackage{multirow}
\usepackage{nicefrac}       
\usepackage[utf8]{inputenc} 
\usepackage{float}
\usepackage{amsfonts}
\usepackage{cleveref}
\usepackage{array}
\usepackage{comment}
\usepackage{enumitem}
\newcolumntype{M}[1]{>{\centering\arraybackslash}m{#1}}

\newcommand*{\defeq}{\mathrel{\vcenter{\baselineskip0.5ex \lineskiplimit0pt
			\hbox{\scriptsize.}\hbox{\scriptsize.}}}%
	=}

\usepackage[utf8]{inputenc}
\usepackage{intcalc}
\usepackage{ifthen}

\begin{document}
	\pagestyle{headings}
	\mainmatter

	\title{Group Pruning using a Bounded-\(\ell_p\) norm for Group Gating and Regularization}
	
	\titlerunning{Group Pruning using Bounded-\(\ell_p\) norm for Group Gating and Regularization}
	\authorrunning{C. K. Mummadi et al.}
	\author{
		Chaithanya Kumar Mummadi\inst{1,2}\orcidID{0000-0002-1173-2720} \and
		Tim Genewein\inst{1}\orcidID{0000-0001-8039-4027}\thanks{Currently at DeepMind} \and
		Dan Zhang\inst{1}\orcidID{0000-0003-0930-9162} \and
		Thomas Brox\inst{2}\orcidID{0000-0002-6282-8861} \and
		Volker Fischer\inst{1}\orcidID{0000-0001-5437-4030} 
	}
	
	\institute{
		$^1$ Bosch Center for Artificial Intelligence, Robert Bosch GmbH, Germany \\
		$^2$ University of Freiburg, Germany
	}

	\maketitle

	\begin{abstract}
    Deep neural networks achieve state-of-the-art results on several tasks while increasing in complexity.
	It has been shown that neural networks can be pruned during training by imposing sparsity inducing regularizers. 
	In this paper, we investigate two techniques for group-wise pruning during training in order to improve network efficiency.
	We propose a gating factor after every convolutional layer to induce channel level sparsity, encouraging insignificant channels to become exactly zero.
	Further, we introduce and analyse a bounded variant of the \(\ell_1\) regularizer, which interpolates between \(\ell_1\) and \(\ell_0\)-norms to retain performance of the network at higher pruning rates.
	To underline effectiveness of the proposed methods,
	we show that the number of parameters of ResNet-164, DenseNet-40 and MobileNetV2 can be reduced down by \(30\%\), \(69\%\), and \(75\%\) on CIFAR100 respectively without a significant drop in accuracy. 
	We achieve state-of-the-art pruning results for ResNet-50 with higher accuracy on ImageNet. 
	Furthermore, we show that the light weight MobileNetV2 can further be compressed on ImageNet without a significant drop in performance.
\end{abstract}

\section{Introduction}
\label{sec:introduction}

Modern deep neural networks are notoriously known for requiring large computational resources, which becomes particularly problematic in resource-constrained domains, such as in automotive, mobile or embedded applications. 
Neural network \emph{compression} methods aim at reducing the computational footprint of a neural network while preserving 
task performance (e.g. classification accuracy) \cite{cheng2017survey,sze2017efficient}.  
One family of such methods, \emph{Network pruning}, operates by removing unnecessary weights or even whole neurons or convolutional featuremaps (``channels'') during or after training, thus reducing computational resources needed at test time or deployment. %
A simple relevance-criterion for pruning weights is weight-magnitude: ``small'' weights contribute relatively little to the overall computation (dot-products and convolutions) and can thus be removed.

However, weight-pruning leads to unstructured sparsity in weight matrices and filters. While alleviating storage demands, it is non-trivial to exploit unstructured sparsity for reducing computational burden during forward-pass operation.  
This effect becomes even more pronounced on today's standard hardware for neural network computation (GPUs), which is typically designed for massively parallel operation. %
In contrast to individual-weight pruning, neuron- and featuremap-pruning allows dropping whole slices of weight matrices or tensors, which straightforwardly leads to a reduction of forward-pass FLOPS, energy consumption as well as on- and off-line memory requirements. However, it is more intricate to determine the relevance of whole neurons/featuremaps than that of weights.

In this paper, we propose and evaluate a method for \emph{group-wise} pruning. A group typically refers to all weights that correspond to a neuron or convolutional filter, but could in principle also be chosen to correspond to
different sub-structures such as larger parts of a layer or even whole blocks/layers in architectures with skip-connections. The central idea of our method is the addition of a ``trainable gate'', that
is a \emph{parameterized, multiplicative factor}, per group. During training, the gate-parameter is learned for each gate individually, allowing the network to learn the relevance of each neuron/featuremap.
After training, groups of low relevance can be straightforwardly identified and pruned without significant loss in accuracy. The resulting highly structured sparsity patterns can be readily used to reduce the size of
weight-matrices or -tensors.
An important aspect of our method is that we use a sparsity-inducing regularizer during training to force a maximally large number of gates towards zero. We empirically compare different choices for this sparsity-inducing regularizer and in addition to previously proposed \(\ell_1\) or \(\ell_2\) norms, we propose and evaluate a smoothened version of the \(\ell_0\) norm (which can also be viewed as a saturating version of an \(\ell_p\) norm). The latter allows for a certain decoupling of parameter-importance and parameter-magnitude, which is in contrast to standard regularizers that penalize parameters of large magnitude regardless of their importance.

\begin{itemize}
	\item We investigate the effect of group pruning using bounded \(\ell_p\) norms for group gating and regularization on different network architectures (LeNet5, DenseNet, ResNet and MobileNetV2) and data-sets (MNIST, CIFAR100, 
	ImageNet) achieving comparable or superior compression and accuracy.
	\item We show that our gating function drives the gating factors to become exactly zero for the insignificant channels during training.
	\item Applying \(\ell_2\) regularizer on our gating parameters, rather than on weights, leads to significant pruning for ResNet and DenseNet without a drop in accuracy and further improves the accuracy of MobileNetV2 on both CIFAR100 and ImageNet.
	\item We also propose a bounded variant of the common \(\ell_1\) regularizer to achieve higher pruning rates and retain generalization performance.
\end{itemize}

	\section{Related work}
\label{sec:existing_work}

\textbf{Neural Network compression}. 
Most approaches in the literature resort to \emph{quantization} and/or \emph{pruning}. 
In this context, quantization refers to the reduction of required bit-precision of computation --- either of weights only \cite{chen2015compressing,gong2014compressing,han2016deep,zhou2017incremental} or both weights and activations \cite{courbariaux2016binarized,rastegari2016xnor,hubara2017quantized,gysel2018ristretto,wu2018training}.
Network pruning attempts to reduce the number of model parameters and is often performed in a single step after training, but some variants also perform gradual pruning during training \cite{frankle2018lottery,han2016deep,han2015learning} or even prune and re-introduce weights in a dynamic process throughout training \cite{han2017dsd,guo2016dynamic}. In contrast to individual weight pruning \cite{han2016deep}, group-pruning methods (pruning entire neurons or feature-maps that involve groups of weights) lead to highly structured sparsity patterns which easily translate into on-chip benefits during a forward-pass \cite{wen2016learning,zhou2016less,alvarez2016learning}.

Pruning and quantization can also be combined \cite{han2016deep,ullrich2017soft,federici2017improved,achterhold2018variational}. Additionally, the number of weights can be reduced \emph{before} training by architectural choices as in SqueezeNet \cite{iandola2016squeezenet} or MobileNets \cite{howard2017mobilenets}. As we show in our experiments, even parameter-optimized architectures such as MobileNets can still benefit from post-training pruning.

\textbf{Relevance determination and sparsity-inducing regularization}.
Many pruning methods evaluate the relevance of each unit (weight, neuron or featuremap) and remove units that do not pass a certain relevance-threshold \cite{han2016deep,guo2016dynamic,han2017dsd,liu2017learning}. Importantly, optimizing the relevance-criterion that is later used for pruning thus becomes a secondary objective of the training process --- in this case via weight-magnitude regularization.  
An undesirable side-effect of \(\ell_1\)- or \(\ell_2\)-weight-decay \cite{hanson1989comparing} when used for inducing sparsity is that important, non-pruned weights still get penalized depending on their magnitude, leading to an entanglement of parameter-importance and magnitude. An ideal sparsity-inducing regularizer would act in an (approximately) binary fashion, similar to how the \(\ell_0\) norm simply counts number of non-zero parameters, but is not affected by the magnitude of the non-zero parameters.
The problem of determining the relevance of model parameters has also been phrased in a Bayesian fashion via \emph{automatic relevance determination} (ARD) and sparse Bayesian learning \cite{karaletsos2015automatic,mackay1995probable,nealbayesian}, which has recently been successfully applied to weight-pruning \cite{molchanov2017variational}, weight ternarization \cite{achterhold2018variational} and neuron-/featuremap-pruning by enforcing group-sparsity constraints \cite{ghosh2018structured,louizos2017bayesian,federici2017improved,neklyudov2017structured}. These methods require (variational) inference over the parameter posterior instead of standard training.

\textbf{Neuron-/featuremap-pruning}.
Determining the importance of neurons or feature maps is non-trivial \cite{wen2016learning,zhou2016less,alvarez2016learning}. Approaches are based on thresholding the norm of convolutional kernels or evaluating activation-statistics. However, both approaches come with certain caveats and shortcomings \cite{molchanov2016pruning,ye2018rethinking}. Some methods try to explicitly remove neurons that do not have much impact on the final network prediction \cite{hu2016network,li2016pruning,molchanov2016pruning}. Other methods propose a more complex optimization procedure with intermediate pruning steps and fine-tuning \cite{he2017channel,luo2017thinet}, such that the non-pruned network can gradually adjust to the missing units.

Our approach is closely related to \cite{louizos2017learning}, who also use trainable, multiplicative gates for neuron-/featuremap-pruning. However, in their formulation gates Bernoulli random variables. Accordingly, learning of their gate parameters is done via (variational) Bayesian inference. In contrast, our method allows network training in a standard-fashion (with an additional regularizer term) without requiring sampling of gate parameters, or computing expected gradients across such samples.
Other closely related works are \cite{liu2017learning,ye2018rethinking}, who induce sparsity on the multiplicative scaling factor $\gamma$ of Batch Normalization layers and later use the magnitude of these factors for pruning channels/featuremaps. Similarly, \cite{huang2017data} use a trainable, linear scaling factor on neurons / featuremaps with an \(\ell_1\)-norm sparsity-inducing regularizer.
We perform experiments to directly compare our method against all the above closely related works.
Additionally, we reimplement the technique proposed by \cite{liu2017learning} and treat it as a baseline to compare our results against it in all experiments.

	\section{Bounded-\(\ell_{p,0}\) norm}
\label{sec:theory}

The \(p\)-norm (a.k.a. \(\ell_p\)-norm) and \(0\)-norm of a vector \(x \in \mathbb{R}^n\) of dimension $n$ are respectively defined as:
\vskip -0.4 cm
\begin{equation}\label{def:norms}
	\|x\|_p \defeq \left(\sum\limits_{i = 1}^{n} |x_i|^p\right)^{1/p}\ \ \ \ \ \ \ \ 
	\|x\|_0 \defeq \sum\limits_{i=1}^{n}(1 - \textbf{1}_0(x_i)).
\end{equation}
Here, \(\textbf{1}_{a}(b)\) being the function which is one iff \(a = b\) and zero otherwise.
While the \(p\)-norms constitute norms in the mathematical sense, the \(0\)-norm (a.k.a. \textit{discrete metric}), does not due to the violation of the triangle inequality. 

It is constant almost everywhere and hence gradient based optimization techniques are unusable. %
We use a differentiable function adapted from~\cite{Weston2003}, which around \(0\) interpolates, controlled by a parameter \(\sigma > 0\), between the \(p\)- and \(0\)-norm: 

\begin{definition}
For \(\sigma > 0\), we call the mapping \(\|.\|_{\text{bound-}p,\sigma}: \mathbb{R}^{n}\rightarrow \mathbb{R}_{+}\) with
\vskip -0.3 cm
\begin{equation}\label{def:boundednorm}
\|x\|_{\text{bound-}p,\sigma} \defeq \sum\limits_{i=1}^{n} 1 - \text{exp}\left(-\frac{|x_i|^p}{\sigma^p}\right)
\end{equation}
the \textbf{bounded-\(\ell_{p,0}\) norm} or \textbf{bounded-\(\ell_{p}\) norm}.
Fig. \ref{fig:figure_illu_norms} illustrates the bounded-\(\ell_{p,0}\) norm with \(p=1,2\) and different $\sigma$. One sees that \(\|x\|_{\text{bound-}p,\sigma}\) is bounded to \([0,n)\) and differentiable everywhere except \(x_i = 0\) for one or more coefficients of \(x\). 
Further, in contrast to the \(0\)-norm, it has a non-zero gradient almost everywhere.
\end{definition}

\begin{lemma}\label{lemma1}
The bounded-\(\ell_{p,0}\) norm has the following properties:
\begin{itemize}[itemsep=-0.5em,topsep=0pt]
	\item For \(\sigma \rightarrow 0^{+}\) the bounded-norm converges towards the \(0\)-norm:

	\begin{equation}\label{lemma1:is0norm}
	\lim\limits_{\sigma \rightarrow 0^{+}} \|x\|_{\text{bound-}p,\sigma} = \|x\|_{0}.
	\end{equation}
	\item In case \(|x_i| \approx 0\) for all coefficients of \(x\), the bounded-norm of \(x\) is approximately equal to the \(p\)-norm of \(x\) weighted by \(1/\sigma\):

	\begin{equation}\label{lemma1:ispnorm}
	\|x\|_{\text{bound-}p,\sigma} \approx \left\lVert\frac{x}{\sigma}\right\rVert_{p}^{p}
	\end{equation}
\end{itemize}
\end{lemma}
\begin{proof}
See \Cref{proof:lemma1} for proof.
\end{proof}

	\section{Methodology}
\label{sec:Methodology} 
With the use of the bounded-\(\ell_{p}\) norm introduced in the previous section, we subsequently present a simple and straightforward technique to perform \emph{group wise pruning} in deep CNNs. Here, \emph{group} is referred to as a set of weights, e.g., a filter in a convolutional layer associated to a feature map or, in case of a fully connected layer, a single target neuron. 
\begin{figure}
\begin{center}
\includegraphics[width=0.7\linewidth, height=0.15\textheight]{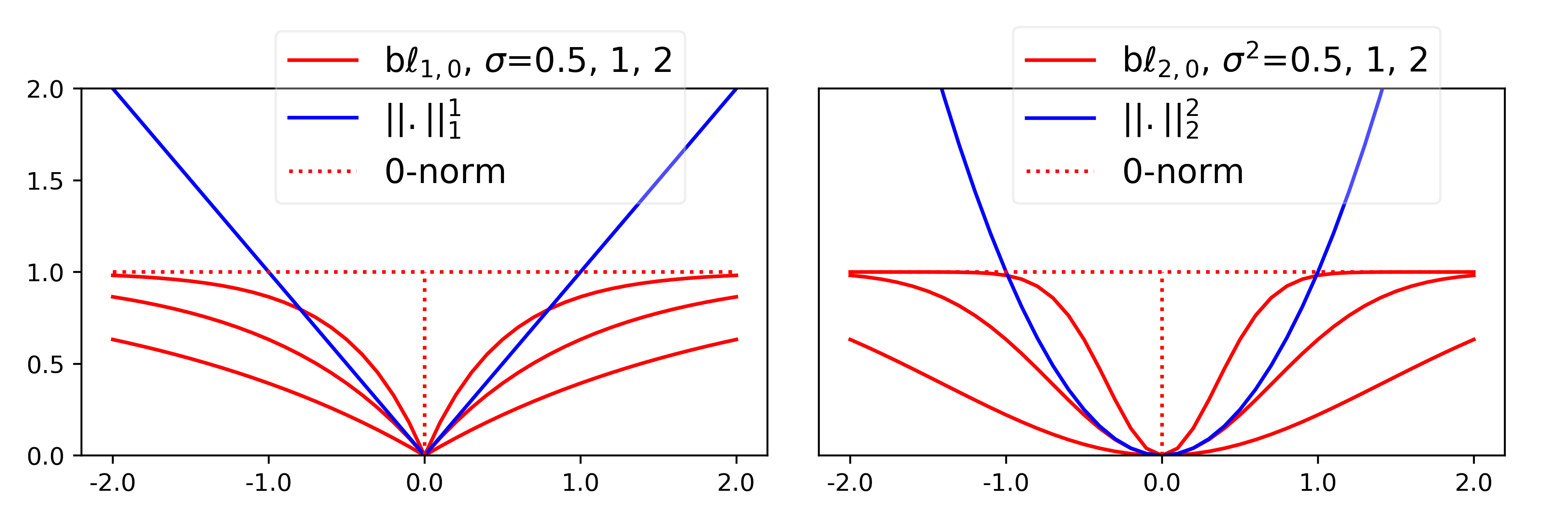}

\caption{Illustration of bounded-\(\ell_{p,0}\) (b\(\ell_{p,0}\)) norms with $p\in\{1,2\}$: Interpolation from \(\ell_1\)-norm to \(0\)-norm (left) and from \(\ell_2\)-norm to \(0\)-norm (right) with different $\sigma$.}

\label{fig:figure_illu_norms}
\end{center}
\end{figure}

\textbf{Bounded-\(\ell_1\) regularizer}: It is a common practice to use sparsity inducing \(\ell_1\) penalty to shrink parameters during training. ~\cite{liu2017learning} has performed channel-wise pruning by imposing \(\ell_1\) penalty on the scaling factor \(\gamma\) of Batch Normalization (BN) layers that correspond to featuremaps in convolutional layers. We denote these scaling factors as \emph{linear gates} in this work. Thus, the total loss $L$ consists of the standard empirical loss $l$ and an additional \(\ell_1\) penalty on the linear gates:
\begin{equation}\label{def:loss_l1_norm}
  \mathcal{L} = \sum_{(x,y)} l\left(f\left(x, W\right), y\right) + \lambda \sum_{\gamma \in G} \left|\gamma\right|
\end{equation}
where \(f\) denotes the deep neural network, \(x,y\) denote training input and target, \(W\) denotes the network weights, \(\gamma\) denotes a single scaling factor from the aggregated set of all linear gates $G$. The \(\ell_1\) regularizer acts upon all linear gates and pushes them towards zero. The channels with linear gates whose magnitude is less than the relevance threshold are then pruned to obtain a narrow network. Here, the linear gates should accomplish two different tasks i) get optimized along the other network parameters to improve the network performance and ii) shrink down towards zero to induce channel level sparsity in the network. The hyperparameter \(\lambda\) defines the strength of the regularizer and controls the trade-off between primary objective and \(\ell_1\) penalty. Increasing \(\lambda\) would yield higher pruning rates at the cost of reduced network performance. The \(\ell_1\) regularizer penalizes each parameter at a same rate
irrespective of its role and importance in accomplishing the primary objective. In general, not all parameters should receive equal penalty. We address this issue by employing a norm as defined in \Cref{def:boundednorm} as a sparsity inducing regularizer with \(p=1\) and denote it as bounded-\(\ell_{1}\) regularizer as it is bounded to \([0,1]\). 

Figure \ref{fig:figure_illu_norms} (left) shows that the bounded-\(\ell_1\) norm is a variant of the normal \(\ell_1\) norm and both penalize larger parameters. 
Importantly for the bounded variant, the penalty on larger weights does not increase as strong as for the normal norm, and only smaller weights are penalized comparably. Larger parameters, for which the bounded variant saturates, become primarily subject to the task loss.
In other words, for the bounded variant, the penalty for large parameters becomes decoupled from the size of the parameters and converges to a constant value whereas for small parameters the penalty is relative large and forces them to even smaller values.
Similar to the \(\ell_{1}\) penalty, the bounded \(\ell_{1}\) norm can be added as a regularization term in the objective function.

\begin{equation}
\label{eq:bounded_l1_reg}
  \mathcal{L}^\ast = \sum_{(x,y)} l\left(f\left(x, W\right), y\right) + \lambda \sum_{\gamma \in G} \left[1 - e^{-\frac{\left|\gamma\right|}{\sigma}}\right]
\end{equation}
The gradient of the parameter \(\gamma\) w.r.t. \(\ell_1\) and bounded-\(\ell_{1}\) regularization equals:

\begin{equation}
  \frac{\partial\mathcal{L}_\mathrm{reg}}{\partial \gamma} = \lambda \cdot \mathrm{sign}\left(\gamma\right), \hspace{0.4cm}
  \frac{\partial\mathcal{L}_\mathrm{reg}^\ast}{\partial \gamma} = \lambda \cdot \mathrm{sign}\left(\gamma\right)\frac{e^{-\frac{\left|\gamma\right|}{\sigma}}}{\sigma}
\end{equation}
The above equations indicate that the \(\ell_1\) norm updates gradients at a scale of \(\lambda\) irrespective of their magnitude. On the other hand, bounded-\(\ell_1\) norm provides no or small gradients for parameters with higher magnitude and large gradients for smaller parameters. In this manner, parameters with larger values receive gradients mainly from the first part of $\mathcal{L}^\ast$, being informative to accomplish the primary classification task. 

Another interesting property of such norm is: The hyperparameter \(\sigma\) scales the regularization strength by controlling the interpolation between the \(\ell_1\)- and \(0\)-norm. As \(\sigma\) gets smaller, the bounded-\(\ell_1\) norm converges to the \(0\)-norm according to \Cref{lemma1}. Larger \(\sigma\) allows regularization of all parameters whereas smaller \(\sigma\) guides the regularizer to penalize only parameters of smaller magnitude while liberating the larger ones. %
Larger values of \(\sigma\) enforce weaker regularization and smaller values enforce stronger regularization (also compare Fig. \ref{fig:figure_illu_norms}). 

Given the behavior of \(\sigma\), we can schedule it by gradually reducing its value during training.  
In doing so, the norm initially regularizes a larger number of parameters and then gradually shrinks down the insignificant ones towards zero while simultaneously filtering out the important ones. We can imagine the scheduling of \(\sigma\) as opening the gates of the \(0\)-norm to make it differentiable which allows the insignificant parameters to fall into the valley of the norm and gradually close the gates to leave out the important parameters. It is fairly straightforward to include the hyperparameter \(\sigma\) also in the case of the \(\ell_1\) norm by replacing \(\left|\gamma\right|\) with \(\frac{\left|\gamma\right|}{\sigma}\) in \Cref{def:loss_l1_norm} but it is similar to scaling the hyperparameter \(\lambda\) to \(\frac{\lambda}{\sigma}\) in this case. The scheduling of \(\sigma\) in \(\ell_{1}\)-norm
increases its regularization strength and pushes down all the parameters towards zero which affects task performance of the network.

\textbf{Bounded-\(\ell_2\) for group gating}: Both the \(\ell_1\) and bounded-\(\ell_{1}\) regularizers bring down the scalar parameters towards zero but never make them exactly zeros (refer Figure \ref{fig:figure_gate_comparison}). This limitation always demands the setting of a relevance threshold to prune the parameters
and then later requires fine-tuning for a number of iterations to stabilize the task performance of the pruned network. To this end, we propose to use the same bounded-\(\ell_p\) norm that is defined in \Cref{def:boundednorm} as an additional layer in the network with \(p=2\) and \(\sigma=1\). To this, we refer to as a gating layer of \emph{exponential gates} (with gating parameters \(g\)) which is placed after every convolutional or fully connected layer or before a BN layer in the network. This layer serves as a multiplicative gating factor for every channel in the preceding convolutional layer. The gating layer has the same number of gates as the number of channels where each gate gets multiplied to an output channel of a convolutional layer. 

\begin{equation}
  x = \mathrm{conv}\left(\emph{input}\right);\quad y_k = x_k \cdot \left(1 - e^{-g^2_k}\right)
\end{equation}
where \(x\) and \(y\) are the output of the convolutional and gating layer respectively and \(k\) indexes the channel of the convolutional layer. %
Since the gates are added as a layer in the network, we train the gating parameters \(g\) together with the network weights \(W\). In contrast to the \emph{linear gates} \(\gamma\) of BN, we impose the penalty only on the parameters \(g\) of \emph{exponential gates}, yielding the loss function: 
\begin{equation}
  \mathcal{L} = \sum_{(x,y)} l\left(f\left(x, g, W\right), y\right) + \lambda \sum_{g \in G} R\left(g\right)
\end{equation}
The first part of the loss function corresponds to the standard empirical loss of the neural network and \(R(.)\) is the penalty term on the gating parameters \(g\) which could be either \(\ell_2\), \(\ell_1\), or the bounded-\(\ell_{1}\) regularizer. Two interesting properties of the \emph{exponential gates} which makes them distinctive from the \emph{linear gates} are
i) its values are bounded to the range \([0,1)\),
ii) the quadratic exponential nature of the gates fused with the regularizer shrink down the outcome of the gates towards zero rapidly.
 
The regularized \emph{exponential gates} which are jointly optimized with the network weights act as a channel selection layer in the network. These gates actively differentiate the insignificant channels from significant ones during the training phase and gradually turn them off without affecting the network's performance.
In Section \ref{sec:Results}, we empirically show that these exponential gating layers assist the regularizers to drive the insignificant channels to become exactly zero and later compress the network after removing such channels.

The \emph{exponential gating layer} can be added to the network with or without BN. 
In case the gating layer is followed by BN, the statistics from the nulled-out channels remains constant across all the mini-batches since the gate is deterministic and gets multiplied to every input sample. Thus, both the running estimates of its computed mean and variance of the BN is zero for the nulled-out channels. The multiplicative scaling factor \(\gamma\) of BN does not show any effect on those channels but its additive bias \(\beta\) might change the zero channels to non-zero. This can be seen as adding a constant to the zero channels which can be easily alleviated by few iterations of fine-tuning the pruned network. In case the gating layers are added to a CNN without BN, we can prune channels in the network without any need of explicit fine-tuning since the insignificant channels become exactly zero after getting multiplied with the gates during the training phase. 
As a final note, the additional \emph{exponential gating layer} increases the number of trainable parameters in the network but these gates can be merged into the weights of the associated convolutional filter after pruning.

In next section, we empirically evaluate the above-proposed techniques to achieve channel level sparsity, namely, i) bounded-\(\ell_{1}\) norm to prune a larger number of parameters and preserve the task accuracy, and ii) additional gating layer in CNNs to support the regularizers to achieve exactly zero channels.

	\section{Experimental Results}
\label{sec:Results}

We demonstrate the significance of both, the \emph{exponential gating layer} and the \emph{bounded-\(\ell_1\)} regularizer, on different network architectures and datasets, i.e., LeNet5-Caffe on MNIST, DenseNet-40, ResNet-164, MobileNetV2 on CIFAR100 and ResNet-50, MobileNetV2 on ImageNet dataset. We refer to Sec.~\ref{sec:experiments} for the experiment details such as data preprocessing, architecture configuration, and hyperparameter selection.
We use the threshold point $10^{-4}$ on the linear gates and threshold zero on the exponential gates to prune the channels.

\textbf{CIFAR100} The results are summarized in Figure \ref{fig:figure_full_comparison}. We compare the trade-off between classification accuracy on test data against the pruning rates obtained from different regularizers and gates.
We report the average results over 3 different runs.
Here, \(\sigma_{constant}\) refers to the hyperparameter \(\sigma\) that is set to a constant value throughout the training process. We also investigated the influence of scheduling \(\sigma\) in case of bounded-\(\ell_1\) regularizer and compared the results against scheduling \(\sigma\) in \(\ell_1\) regularizer.

\begin{figure}[t!]

		\includegraphics[width=1.0\linewidth, height=0.195\textheight]{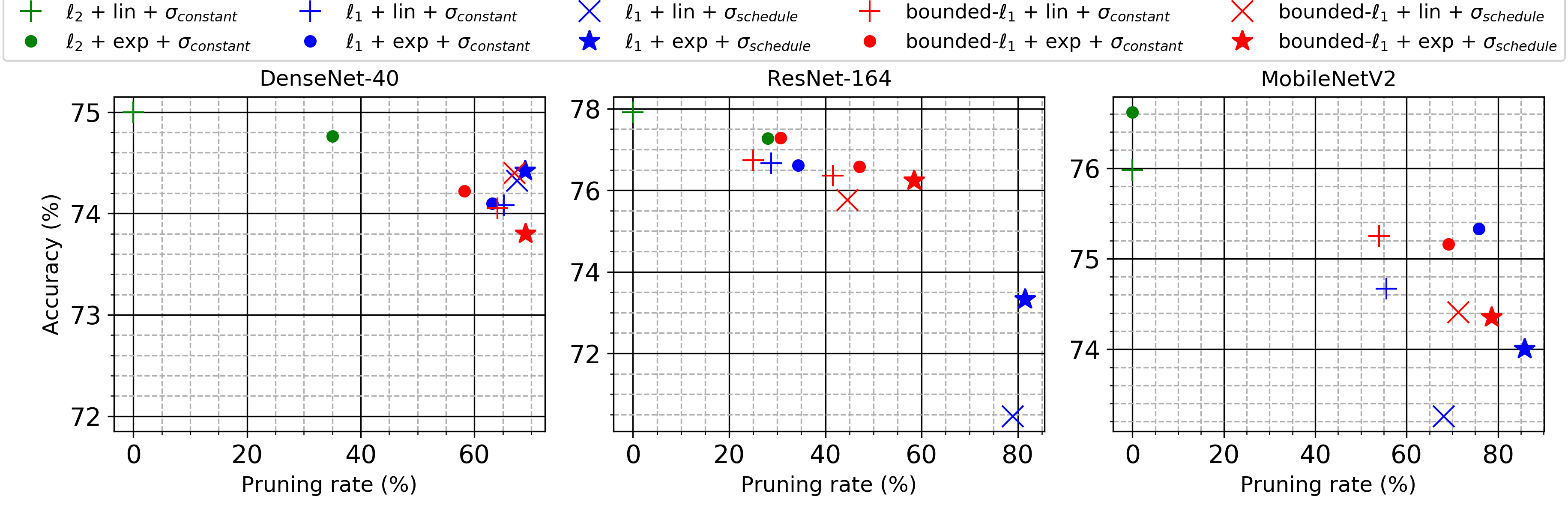}

		\caption{Comparing trade-off between pruning rates and accuracies of different regularizers \(\ell_2\), \(\ell_1\) and bounded-\(\ell_{1}\) with different gates (linear, exponential) at constant and scheduled $\sigma$ on DenseNet-40, ResNet-164 and MobileNetV2 on CIFAR100. In DenseNet-40, the scheduled \(\ell_{1}\) regularizer on exponential gate achieves slightly higher pruning and accuracy rate than the other methods. In ResNet-164, two identical markers represent settings with different regularization strengths. Here, bounded-\(\ell_1\) on exponential gate achieves higher pruning rates with approximately same line of accuracy with other methods. In MobileNetV2, bounded-\(\ell_1\) on linear gate outperforms \(\ell_1\) on linear gate in terms of accuracy with approximately similar pruning rate for both the cases of $\sigma$ (constant and scheduled). However, \(\ell_1\) on exponential gate with constant $\sigma$ preserves the accuracy with higher pruning rate. Thus, the networks with exponential gating layers has higher pruning rates than the linear gates with the accuracy close to baseline. On the other hand, bounded-\(\ell_1\) improves accuracy on linear gates in MobileNetV2 and on both gates in ResNet-164 when compared with \(\ell_1\) regularizer.
		}
		\label{fig:figure_full_comparison}
\end{figure}
\begin{figure}[t!]
	\begin{center}
	\includegraphics[width=0.95\linewidth, height=0.4\textheight]{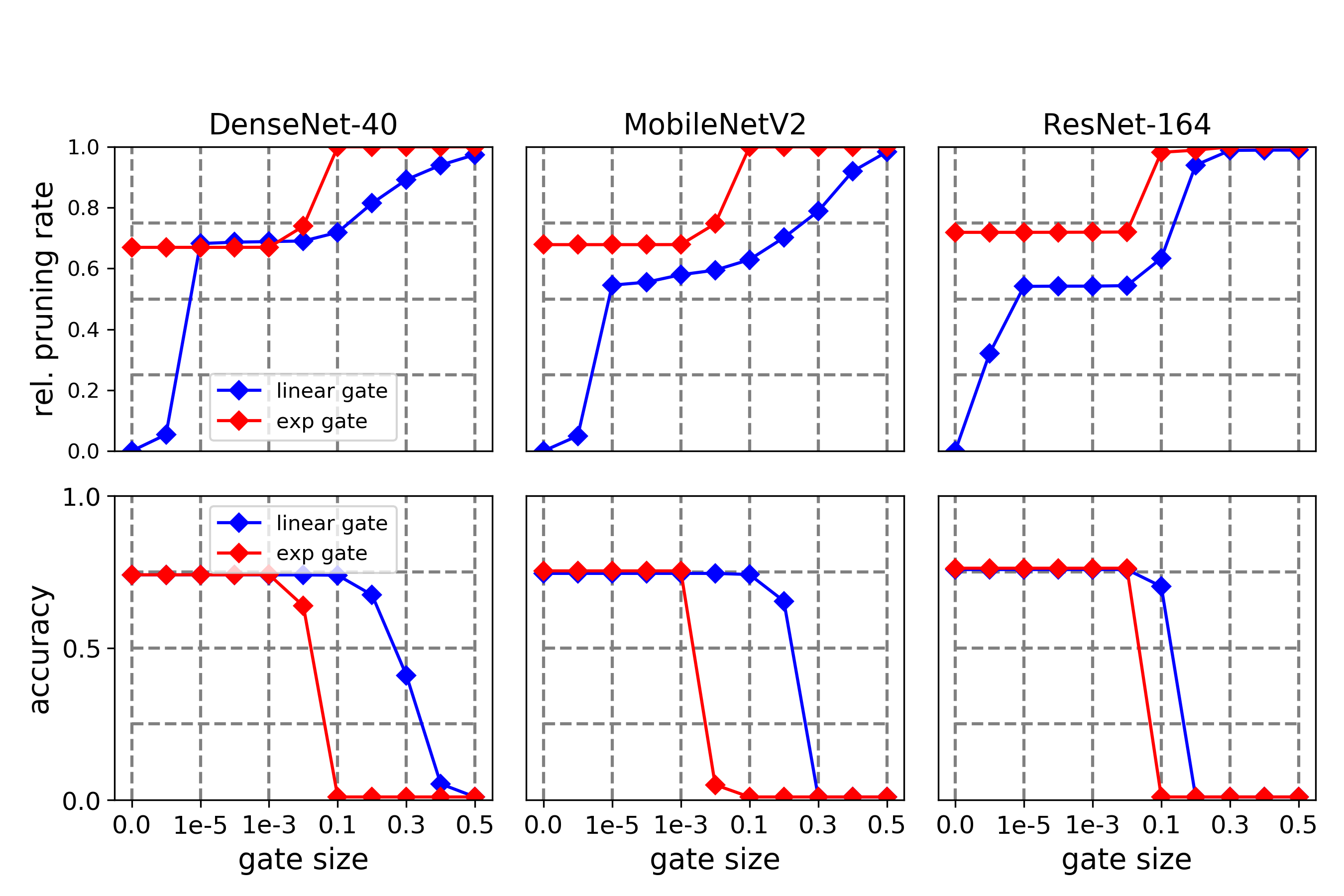}

	\caption{Comparison of pruning rates (top row) and accuracies (bottom row)  on CIFAR100 over different threshold points between linear gate (blue) and the exponential gate (red), both of which are applied in combination with \(\ell_{1}\) regularization. Units that do not pass the threshold on the gate values $|g|$ for the linear gates and $(1 - e^{-g^2})$ for exponential gates are pruned.  Here, pruning rate of networks with exponential gates is superior or comparable to the linear gate at different threshold values. With the exponential gates, the achieved best pruning rates are insensitive to the selection of the threshold within the range $[0,10^{-3}]$. In particular, the threshold zero being (nearly) optimum indicates that the exponential gates can exactly zeroing out removable channels. This observation also holds when combining with the \(\ell_{2}\) and bounded-\(\ell_{1}\) regularizers.}

	\label{fig:figure_gate_comparison}
\end{center}
\end{figure}

From Figure \ref{fig:figure_full_comparison}, it can be seen that the
bounded-\(\ell_{1}\) regularizer on the linear gates results in a higher pruning rate with an accuracy comparable to the \(\ell_{1}\) regularizer in ResNet-164 and provides a higher accuracy than the \(\ell_{1}\) regularizer in MobileNetV2. 
On the other hand, the addition of exponential gating layers in ResNet-164 and MobileNetV2 greatly increases the pruning rates and accuracy upon the linear gate. The bounded-\(\ell_1\) regularizer further improves the accuracy of ResNet-164 with exponential gating layer to 77.28\% and 76.58\% at different regularization strengths with pruning rates 30.73\% and 47\% respectively.
In case of MobileNetV2, \(\ell_1\) on exponential gating layer results pruning rate of 75.83\% with an accuracy 75.33\%.

In contrast to the other networks, the pruning results of bounded-\(\ell_1\) regularizer and exponential gating layer in DenseNet-40 are identical to the results of the \(\ell_1\) regularizer on linear gate. However, the addition of exponential gating layer when combined with the \(\ell_2\) regularizer encourages channel pruning with a marginal drop in performance in both ResNet-164 and DenseNet-40 architectures, whereas the gate improves the classification performance in case of MobileNetV2.
We can also observe that scheduling \(\sigma\) for the \(\ell_1\) regularizer significantly drops the accuracy and increases the pruning rate in both MobileNetV2 and ResNet-164. Scheduling in the bounded-\(\ell_1\) regularizer also increases the pruning rate while retaining the accuracy close to the baseline margin. In MobilenetV2, scheduling the regularizer in bounded-\(\ell_1\) yields higher accuracy and pruning rate on linear gate when compared to the scheduled \(\ell_1\) regularizer. In ResNet-164, the pruning rate raises from 47\% to 58.5\% with an accuracy drop from 76.58\% to 76.23\% in case of exponential gate with scheduled bounded-\(\ell_1\) regularizer.
On the other hand, the impact of the scheduler remains comparable, for both the regularizers in DenseNet-40 and scheduling \(\ell_1\) regularizer on exponential gate increases the pruning rate to 69\% with 74.42\% accuracy.

We compare pruning rates between linear and exponential gates and their accuracy trade-off at different threshold points in Figure \ref{fig:figure_gate_comparison}. We prune channels with gate values less or equal to the threshold and further fine-tune the network for a maximum of three epochs. Across the three different architectures both gates maintain the same accuracy until a critical threshold. The pruning rate of the exponential gates are significantly larger than the linear gates in ResNet-164, MobileNetV2 and comparable in DenseNet-40. In particular, the magnitude of non-zero exponential gates lies in \([10^{-3},0.1]\) and pruning at the threshold larger than $10^{-3}$ removes all channels in the network. Below $10^{-3}$ the exponential gates achieve the optimum performance, i.e., largest pruning rate without loss of the classification accuracy. It is noted that the threshold zero is attainable, indicating that exponential gates can exactly null out removable channels. On contrary, the linear ones gate them with a sufficiently small value (about $10^{-5}$ in the case of Figure \ref{fig:figure_gate_comparison}), thereby necessitating the search of a precise pruning threshold.

\textbf{MNIST} We also test our method on the MNIST dataset using the \emph{LeNet5-Caffe} model. We compare our results with \(\ell_{0}\) regularization from \cite{louizos2017learning}. 
We present different models that are obtained from different regularizers and the weight decay is set to be zero when using the \(\ell_{1}\) or bounded-\(\ell_{1}\) regularization. 
From the results shown in Table \ref{table:MNIST}, it can be observed that network with \emph{exponential gating layer} on different regularizers yield more narrow models than the previous method with lower test errors.

\begin{table}[t!]
\caption{Comparing pruning results of architecture \emph{LeNet-5-Caffe 20-50-800-500} on MNIST dataset from different regularizers like \(\ell_{0}\) from \cite{louizos2017learning} and \(\ell_{1}\), \(\ell_{2}\), Bounded-\(\ell_{1}\) on the network with \emph{exponential gating layers}. We show the resulting architectures obtained from different pruning methods and their test error rate. It can be seen that our architectures are narrower than the one from previous method with comparable or smaller test error rates.}
\vskip -0.05in
\label{table:MNIST}
\begin{center}
\begin{small}
\begin{tabular}{l l l l l l l l l}
\toprule
Method & & & & Pruned architecture & & & Error(\(\%\)) &  \\
\midrule 
\addlinespace
\(\ell_{0}\), \cite{louizos2017learning} & & & &  20-25-45-462 & & & 0.9 \\
\addlinespace
\(\ell_{0}\), \cite{louizos2017learning} & & & &  9-18-65-25 & & & 1.0 \\
\addlinespace
\(\ell_{2}\), \(\lambda_2\) = 5e-4 & & & &  8-19-117-24 & & & 0.79 \\
\addlinespace
\(\ell_{1}\), \(\lambda_1\) = 1e-3 & & & &  8-13-37-25 & & & 0.98 \\
\addlinespace
bounded-\(\ell_{1}\), \(\lambda_1\) = 4e-3 & & & &  9-17-43-25 & & & 0.92 \\
\addlinespace
bounded-\(\ell_{1}\), \(\lambda_1\) = 3e-3 & & & &  9-20-54-27 & & & 0.67 \\

\bottomrule
\end{tabular}
\end{small}
\end{center}
\end{table}

\textbf{ImageNet} We also present pruning results of ResNet-50 and MobileNetV2 for the ImageNet dataset. 
On ResNet-50, we primarily investigate the significance of the \emph{exponential gating layer} with \(\ell_{1}\) and \(\ell_{2}\) regularization. 
From Figure \ref{fig:Resnet50_plot}, it can be seen that the \emph{exponential gating layer} combined with the \(\ell_{2}\) regularizer outperforms ResNet-101(v1) from~\cite{ye2018rethinking} in terms of pruning rate and accuracy. 
\(\ell_{1}\) regularization further penalizes the gating parameters and achieves 39\% and 73\% sparsity in the network with a drop of 1.3\% and 5.4\% Top-1 accuracy respectively at different regularization strengths. 
We compare these results and show that our method prunes more parameters than the previous pruning methods with the same line of accuracy.

\begin{figure}[t!] 
	\begin{center}
	\includegraphics[width=0.7\linewidth, height=0.3\textheight]{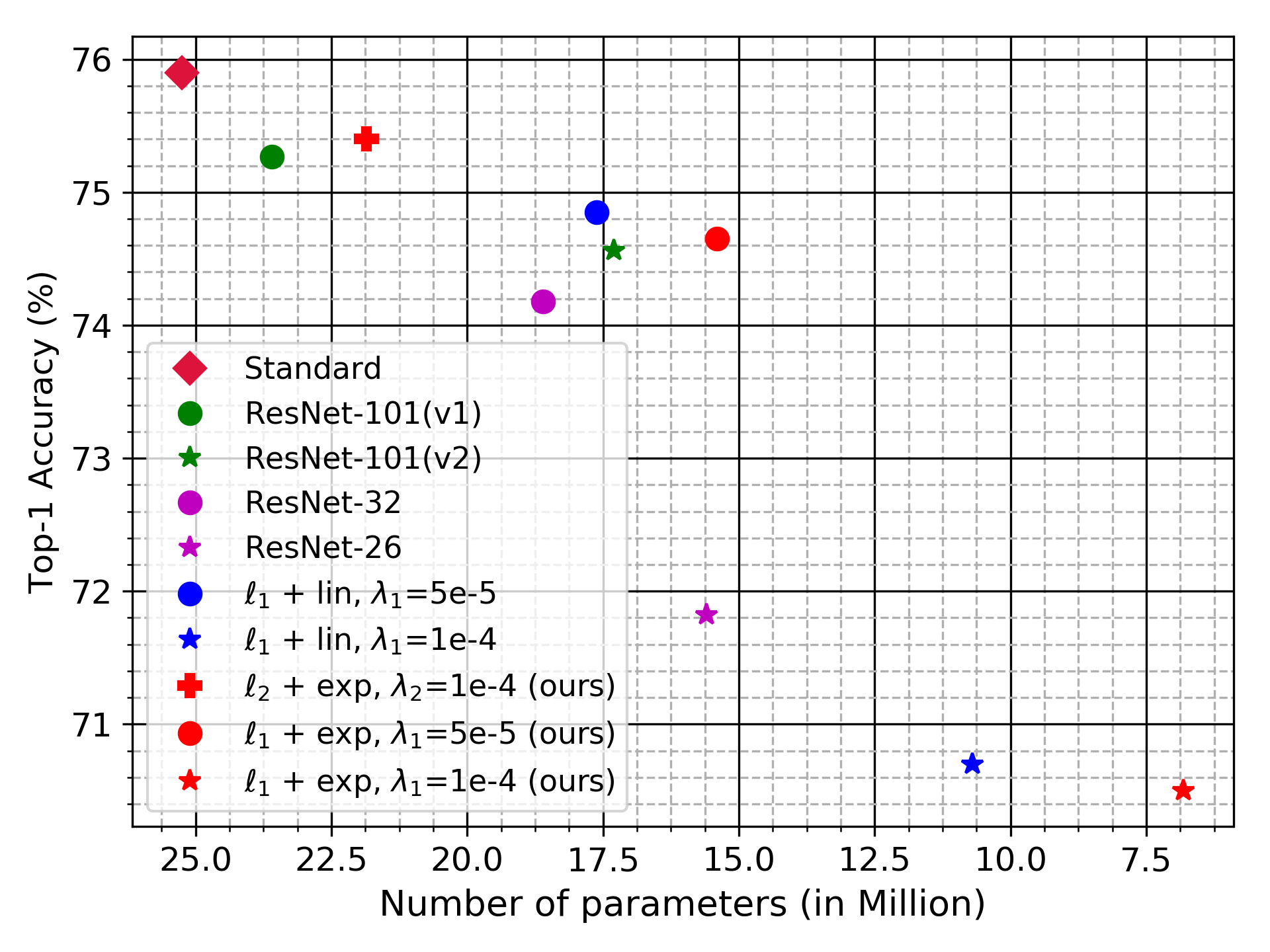}
	\caption{Comparing our pruning results of ResNet-50 (\(\ell_{2}\) \& \(\ell_{1}\) on \emph{exponential gating layer}, \(\ell_{2}\) + $exp$ and \(\ell_{1}\) + $exp$) on ImageNet dataset against the previous methods like ResNet-101(v1) and ResNet-101(v2) from ~\cite{ye2018rethinking}, ResNet-32 and ResNet-26 which are obtained from block pruning on ResNet-50 ~\cite{huang2017data}. We also compared our results against the method \emph{\(\ell_1\) on linear gate} (\(\ell_{1}\) + $lin$) from \cite{liu2017learning} by implementing it on ResNet-50. Here, 'Standard' refers to the baseline model without pruning. We show the trade-off between top1-accuracy and number of remaining parameters of the network from different methods. It can be seen that the network from \(\ell_{2}\) + $exp$ has about similar accuracy as ResNet-101(v1) but prunes 1.75M more parameters than the latter. Similarly, network from \(\ell_{1}\) + $exp$ with \(\lambda_{1}=5e\)-$5$ and $1e$-$4$ prunes 2M and 3.5M more parameters respectively and has comparable accuracy with the other methods.}
	\label{fig:Resnet50_plot}
\end{center}
\end{figure}
\begin{table}[t!]
\caption{Results on MobileNetV2 trained for \(100\) epochs on ImageNet. \emph{Bounded-\(\ell_1\) on linear gate} achieves higher accuracy than \emph{\(\ell_1\) on linear gate} and closer to the standard training with reduced number of parameters. On the other hand, exponential gate with \(\ell_1\) regularizer reduces number of parameters without a significant drop in accuracy and improves accuracy when combined with \(\ell_2\) regularizer. Here M stands for Millions.}
\vskip 0.1in
\begin{small}

\begin{tabular}{lllllllll}
\toprule

Network- MobileNetV2  & & Top-1 \% & & Top-5 \% & & \#Params & & \#FLOPS\\
\midrule 
\addlinespace
Standard training \(\lambda_2\)=1e-5 & & 70.1 & & 89.25 & & 3.56 M  & & 320.2 M  \\
\addlinespace
\(\ell_1+ lin\), \(\lambda_1\) = 5e-5, \(\lambda_2\) = 1e-5 & & 69.54 & & 89.14 & & 3.37 M & & 275.0 M  \\
\addlinespace
bounded-\(\ell_1+ lin\), \(\lambda_1\) = 5e-5, \(\lambda_2\) = 1e-5 & & 69.9 & & 89.17 & & 3.40 M & & 280.0 M  \\
\addlinespace
\(\ell_2+ exp\), \(\lambda_2\) = 4e-5 & & \textbf{70.7} & & 90.0 &  & 3.56 M & & 312.8 M  \\
\addlinespace
\(\ell_1+ exp\), \(\lambda_1\) = 5e-5, \(\lambda_2\) = 4e-5 & & 69.9 & & 89.438 & & \textbf{3.00} M & & 280.0 M  \\

\bottomrule
\end{tabular}
\end{small}
\label{table:mobilenet_imagenet}
\end{table}

On MobileNetV2, we compare \(\ell_1\) against bounded-\(\ell_1\) on linear gate and \(\ell_1\) against \(\ell_2\) on exponential gate. 
From Table \ref{table:mobilenet_imagenet}, it can be observed that the \emph{bounded-\(\ell_1\) on linear gate} achieves higher accuracy than its counterpart \emph{\(\ell_1\) on linear gate} with a slightly higher number of parameters. 
On the other hand, \(\ell_1\) penalty on exponential gate prunes a larger number of parameters and approximately keeps the accuracy of standard training whereas \(\ell_2\) on exponential gate improves the Top-1 accuracy by 0.6\%.

	\section{Conclusion}
\label{sec:Conclusion}
In this work, we propose a straightforward and easy to implement novel approach for group-wise pruning of DNNs. We introduce \emph{exponential gating layers}, which learn importance of the channels during training and drive the insignificant channels to become exactly zero. Additionally, we propose \emph{bounded-\(\ell_1\)} regularization to penalize the gate parameters based on their magnitude. Different combinations of these techniques (gating functions and regularizers) are evaluated for a set of common DNN architectures for image classification. We found that the combination of exponential gating function with an \(\ell_1\) or its bounded variant is superior than the other approaches (cf. Fig. \ref{fig:figure_full_comparison}). Finally, these techniques result in higher compression rates with accuracy comparable to existing pruning approaches on ImageNet (cf. Fig. \ref{fig:Resnet50_plot}).

	\clearpage

	\bibliographystyle{splncs04}
	\bibliography{bibtex_references}
	\clearpage

	\thispagestyle{empty}

\begin{center}
	{\rule{\textwidth}{0.4pt}}
	\large\text{}\\
	\Large\textbf{Group Pruning using a Bounded-\(\ell_p\) norm for Group Gating and Regularization}\\
	\normalsize\text{}\\
	{\rule{\textwidth}{0.4pt}}
	\large\text{}\\
	\vspace{2ex}
\end{center}
\section*{Supplementary material}

\renewcommand{\thetable}{A\arabic{table}}
\renewcommand{\thefigure}{A\arabic{figure}}
\renewcommand{\thesection}{A\arabic{section}}
\setcounter{table}{0}
\setcounter{figure}{0}
\setcounter{section}{0}
\setcounter{equation}{0}
\setcounter{page}{1}

\section{Proof of Lemma 1 (\Cref{lemma1}):}\label{proof:lemma1}

To improve readability, we will restate Lemma 1 from the main text:

The mapping \(\|.\|_{\text{bound-}p,\sigma}\) has the following properties:

\begin{itemize}
	\item For \(\sigma \rightarrow 0^{+}\) the bounded-norm converges towards the \(0\)-norm:
	\begin{equation}\label{proof:is0norm}
		\lim\limits_{\sigma \rightarrow 0^{+}} \|x\|_{\text{bound-}p,\sigma} = \|x\|_{0}.
	\end{equation}
	\item In case \(|x_i| \approx 0\) for all coefficients of \(x\), the bounded-norm of \(x\) is approximately equal to the \(p\)-norm of \(x\) weighted by \(1/\sigma\):
	\begin{equation}\label{proof:ispnorm}
		\|x\|_{\text{bound-}p,\sigma} \approx \left\lVert\frac{x}{\sigma}\right\rVert_{p}^{p}
	\end{equation}
\end{itemize}

\paragraph{Proof:} The first statement \Cref{proof:is0norm} can easily be seen using:
\[
	\lim\limits_{\sigma \rightarrow 0} \text{exp}(-\frac{|x_i|^p}{\sigma^p}) = \textbf{1}_0(x_i)
\]
For the second statement \Cref{proof:ispnorm} we use the taylor expansion of \(\text{exp}\) around zero to get:
\begin{equation}
\begin{split}
	\|x\|_{\text{bound-}p,\sigma} & = \sum\limits_{i=1}^{n} 1 - \text{exp}\left(-\frac{|x_i|^p}{\sigma^p}\right) \\
								 & = \sum\limits_{i=1}^{n} 1 - \sum\limits_{j=0}^{\infty}\left(-\frac{|x_i|^p}{\sigma^p}\right)^j \frac{1}{j!}
\end{split}
\end{equation}
For \(|x_i| \approx 0\) we keep only the leading coefficient \(j = 1\) yielding:
\[
	\|x\|_{\text{bound-}p,\sigma} \approx \sum\limits_{i=1}^{n} \frac{|x_i|^p}{\sigma^p} = \left\lVert\frac{x}{\sigma}\right\rVert_{p}^{p}.
\]

\section{Experiment Details}\label{sec:experiments}

Both CIFAR100 and ImageNet datasets are augmented with standard techniques like random horizontal flip and random crop of the zero-padded input image and further processed with mean-std normalization. The architecture MobileNetV2 is originally designed for the task of classification on ImageNet dataset. We adapt the network\footnote{We changed the average pooling kernel size from \(7\times7\) to \(4\times4\) and the stride from 2 to 1 in the first convolutional layer and also in the second block of bottleneck structure of the network.} to fit the input resolution \(32\times32\) of CIFAR100. ResNet-164 is a pre-activation ResNet architecture containing 164 layers with bottleneck structure while DenseNet with 40 layer network and growth rate 12 has been used. All the networks are trained from scratch (weights with random initialization and bias is disabled for all the convolutional and fully connected layers) with a hypeparameter search on regularization strengths \(\lambda_1\) for \(\ell_1\) or bounded-\(\ell_1\) regularizers and weight decay \(\lambda_2\) on each dataset. The scaling factor $\gamma$ of BN is initialized with 1.0 in case of \emph{exponential gate} while it is initialized with 0.5 for \emph{linear gate} as described in \cite{liu2017learning} and bias $\beta$ to be zero. The hyperparameter $\sigma$ in bounded-\(\ell_1\) regularizer is set to be $1.0$ when the scheduling of this parameter is disabled. All the gating parameters $g$ are initialized with $1.0$.

We use the standard categorical cross-entropy loss and an additional penalty is added to the loss objective in the form of weight decay and sparsity induced \(\ell_1\) or bounded-\(\ell_1\) regularizers.
Note that \(\ell_1\) and bounded-\(\ell_1\) regularization acts only on the gating parameters $g$ whereas weight decay regularizes all the network parameters including the gating parameters $g$.
We reimplemented the technique proposed in ~\cite{liu2017learning} which impose \(\ell_1\) regularization on scaling factor $\gamma$ of Batch Normalization layers to induce channel level sparsity. We refer this method as \emph{\(\ell_1\) on linear gate} and compare it against our methods \emph{bounded-\(\ell_1\) on linear gate}, \emph{\(\ell_1\) on exponential gate} and \emph{bounded-\(\ell_1\) on exponential gate}. We train ResNet-164, DenseNet-40 and ResNet-50 for \(240\), \(130\) and \(100\) epochs respectively. Furthermore, learning rate of ResNet-164, DenseNet-40 and ResNet-50 is dropped by a factor of \(10\) after \((30, 60, 90)\), \((120, 200, 220)\), \((100, 110, 120)\) epochs. The networks are trained with batch size 128 using the SGD optimizer with initial learning rate 0.1 and momentum 0.9 unless specified. Below, we present the training details of each architecture individually.
\\
\newline
\emph{LeNet5-Caffe}: 
Since this architecture does not contain Batch Normalization layers, we do not compare our results with the method \emph{\(\ell_1\) on linear gate}. We train the network with \emph{exponential gating layers} that are added after every convolution/fully connected layer except the output layer and apply different regularizers like \(\ell_1\), bounded-\(\ell_1\) and weight decay separately to evaluate their pruning results. We set the weight decay to zero when training with \(\ell_1\) or bounded-\(\ell_1\) regularizers. The network is trained for 200 epochs with the weight decay and 60 epochs in case of other regularizers.\\
\newline
\emph{ResNet-50}: We train the network with \emph{exponential gating layers} that are added after every convolutional layer on ImageNet dataset. We evaluate performance of the network on different values of regularization strength $\lambda_1$ like $10^{-5}$, $5\times10^{-5}$ and $10^{-4}$. The weight decay $\lambda_2$ is enabled for all the settings of $\lambda_1$ and set to be $10^{-4}$. We analyzed the influence of \emph{exponential gate} and compared against the existing methods. \\
\newline
\emph{ResNet-164}: We use a dropout rate of \(0.1\) after the first Batch Normalization layer in every Bottleneck structure. Here, every convolutional layer in the network is followed by an \emph{exponential gating layer}.\\
\newline
\emph{DenseNet-40}: We use a dropout of 0.05 after every convolutional layer in the Dense block. Here, the \emph{exponential gating layer} is added after every convolutional layer in the network except the first convolutional layer.\\
\newline 
\emph{MobileNetV2}: 
On CIFAR100, we train the network for 240 epochs where learning rate drops by 0.1 at 200 and 220 epochs. A dropout of 0.3 is applied after the global average pooling layer. On ImageNet, we train this network for 100 epochs which is in contrast to the standard training of 400 epochs. We start with learning rate 0.045 and reduced it by 0.1 at 30, 60 and 90 epochs.
We evaluate performance of the network on \emph{exponential gate} over the \emph{linear gate} with \(\ell_1\) regularizer and also tested the significance of bounded-\(\ell_1\) on \emph{linear gate}. \emph{Exponential gating layer} is added after every standard convolutional/depthwise separable convolutional layer in the network.

On CIFAR100, we investigate the influence of weight decay, \(\ell_1\) and bounded-\(\ell_1\) regularizers, the role of \emph{linear} and \emph{exponential gates} on every architecture. We also study the influence of scheduling $\sigma$ in both  \(\ell_1\) and bounded-\(\ell_1\) regularizers on this dataset. For MobileNetV2, we initialize $\sigma$ with $2.0$ and decay it at a rate of $0.99$ after every epoch. In case of ResNet-164 and DenseNet-40, we initialize the hyperparameters $\lambda_1$ and $\lambda_2$ with $10^{-4}$ and $5\times10^{-4}$ respectively and $\sigma$ with $2.0$. We increase the $\lambda_1$ to $5\times10^{-4}$ after 120 epochs and $\sigma$ drops by 0.02 after every epoch until the value of $\sigma$ reaches to 0.2 and later decays at a rate of 0.99.

\end{document}